\documentclass[oneside,english,notitlepage]{article}
\usepackage[T1]{fontenc}
\usepackage[latin9]{inputenc}
\usepackage{verbatim}
\usepackage{float}
\usepackage{amsthm}
\usepackage{amsmath}
\usepackage{amssymb}

\makeatletter

\newcommand{\noun}[1]{\textsc{#1}}
\floatstyle{ruled}
\newfloat{algorithm}{tbp}{loa}
\providecommand{\algorithmname}{Algorithm}
\floatname{algorithm}{\protect\algorithmname}

 \theoremstyle{definition}
  \newtheorem{example}{\protect\examplename}
  \theoremstyle{definition}
  \newtheorem{problem}{\protect\problemname}
  \theoremstyle{definition}
  \newtheorem{defn}{\protect\definitionname}
\theoremstyle{plain}
\newtheorem{thm}{\protect\theoremname}
  \theoremstyle{plain}
  \newtheorem{lem}{\protect\lemmaname}

\usepackage{todonotes}
\usepackage{algorithmicx}
\usepackage{algpseudocode}
\usepackage{thmtools}
\algnewcommand{\LineComment}[1]{\State \(\triangleright\) #1}

\newtheorem{counter}{Counterexample}

\makeatother

\usepackage{babel}
  \providecommand{\definitionname}{Definition}
  \providecommand{\examplename}{Example}
  \providecommand{\lemmaname}{Lemma}
  \providecommand{\problemname}{Problem}
\providecommand{\theoremname}{Theorem}

\begin{document}

\title{Even more generic solution construction in Valuation-Based Systems}

\author{Jordi Roca-Lacostena\and Jesus Cerquides %
\thanks{IIIA - CSIC, Campus UAB, Spain, email: \{jroca, cerquide\}@iiia.csic.es%
}}
\maketitle
\begin{abstract}
Valuation algebras abstract a large number of formalisms for automated
reasoning and enable the definition of generic inference procedures.
Many of these formalisms provide some notions of solutions. Typical
examples are satisfying assignments in constraint systems, models
in logics or solutions to linear equation systems. 

Recently, formal requirements for the presence of solutions and a
generic algorithm for solution construction based on the results of
a previously executed inference scheme have been proposed \cite{Pouly2011a,Pouly2011c}.
Unfortunately, the formalization of Pouly and Kohlas relies on a theorem
for which we provide a counter example. In spite of that, the mainline
of the theory described is correct, although some of the necessary
conditions to apply some of the algorithms have to be revised. To
fix the theory, we generalize some of their definitions and provide
correct sufficient conditions for the algorithms. As a result, we
get a more general and corrected version of the theory presented at
\cite{Pouly2011a,Pouly2011c}.
\end{abstract}

\section{Introduction}

Solving discrete optimization problems is an important and well-studied
task in computer science. One particular approach to tackle them is
known as dynamic programming \cite{Bertele1972} and can be found
in almost every handbook about algorithms and programming techniques.
The works of Bellman \cite{Bellman1957}, Nemhauser \cite{Nemhauser1966}
and Bertelè and Brioschi \cite{Bertele1972} present non-serial dynamic
programming as an algorithm for optimization problems for functions
taking values in the real numbers. A more general approach was taken
by Mitten \cite{Mitten1964a} and further generalized by Shenoy in
1996 \cite{Shenoy1996}, for functions taking values in any ordered
set $\Delta$. Shenoy introduces a set of axioms that later on will
be known as valuation algebras. In those terms, Shenoy is the first
one to connect the concept of solution with the marginalization operation
of the valuation algebra. 

In 2011, Pouly and Kohlas \cite{Pouly2011a,Pouly2011c} drop the assumption
that valuations are functions that map tuples into a value set $\Delta.$
They present several algorithms, and characterize the sufficient conditions
for its correctness. Pouly and Kohlas' algorithms are more general
than their predecessors in the literature. This increased generality
comes at no computational cost, since when applied in the previously
covered scenarios, their particularization coincides exactly with
the previously proposed algorithm. Furthermore, by dropping the assumption
that valuations are functions, their algorithms can be applied to
previously uncovered cases such as the solution of linear equation
systems or the algebraic path problem. Unfortunately, one of the fundamental
results in Pouly and Kohlas' theory is incorrect. 

The contributions of this work are:
\begin{enumerate}
\item We provide a counterexample that invalidates Pouly and Kohlas' results.
\item We generalize the problem solved by Pouly and Kohlas, and provide
and algorithm to solve it.
\item We provide a new sufficient condition for the correctness of the algorithm.
\end{enumerate}
These results provide the most general theory for dynamic programming
up-to-date.
\begin{quotation}
\end{quotation}

\section{Background}

 In this section we start by defining valuation algebras. Later on,
we introduce the marginalization problem and finally we review the
\noun{Collect} algorithm to solve that problem.

The basic elements of a valuation algebra are so-called \emph{valuations},
that we subsequently denote by lower-case Greek letters such as $\phi$
or $\psi.$ Let $D$ be a lattice\cite{Gratzer2011} with a partial
order $\leq,$ two operations meet $\wedge$ and join $\vee,$ a top
element $\top,$ and a bottom element $\bot.$ Given a set of valuations
$\Phi,$ and a lattice of domains $D$, a valuation algebras has three
operations:
\begin{enumerate}
\item \emph{Labeling: $\Phi\rightarrow D;\phi\mapsto d(\phi),$}
\item \emph{Combination: $\Phi\times\Phi\rightarrow\Phi;(\phi,\psi)\mapsto\phi\otimes\psi,$}
\item \emph{Projection: $\Phi\times D\rightarrow\Phi;(\phi,x)\mapsto\phi^{\downarrow x}$
}for \emph{$x\leq d(\phi).$}
\end{enumerate}
satisfying the following axioms:
\begin{description}
\item [{A1}] \emph{Commutative semigroup}: $\Phi$ is associative and commutative
under $\otimes.$
\item [{A2}] \emph{Labeling: }For $\psi,\phi\in\Phi,$ $d(\phi\otimes\psi)=d(\phi)\vee d(\psi).$
\item [{A3}] \emph{Projection: }For $\phi\in\Phi,$ $x\in D,$ and $x\leq d(\phi),$
$d(\phi^{\downarrow x})=x.$
\item [{A4}] \emph{Transitivity: }For $\phi\in\Phi$ and $x\leq y\leq d(\phi),$
$(\phi^{\downarrow y})^{\downarrow x}=\phi^{\downarrow x}.$
\item [{A5}] \emph{Combination: }For $\phi,\psi\in\Phi$ with $d(\phi)=x$,
$d(\psi)=y$, and $z\in D$ such that $x\leq z\leq x\vee y,$ $(\psi\otimes\phi)^{\downarrow z}=\phi\otimes\psi^{\downarrow z\wedge y}.$
\item [{A6}] \emph{Domain: }For $\phi\in\Phi$ with $d(\phi)=x,$ $\phi^{\downarrow x}=\phi.$
\end{description}
We say that valuation $e\in\Phi$ is an identity valuation provided
that $d(e)=\bot$ and $\phi\otimes e=\phi$ for each $\phi\in\Phi.$
As proven in \cite{Pouly2011a}, any valuation algebra that does not
have and identity valuation can easily be extended to have one. In
the following and without loss of generality we assume that our valuation
algebra has an identity valuation $e.$

\paragraph{Variable systems, frames and tuples.}

In most practical applications of valuation algebras, the domains
of the valuations are subsets of a given set of variables $V$. It
is well known (e.g. \cite{Davey2008} page 36) that for any set $V,$
the ordered set $\langle\mathcal{P}(V),\subseteq\rangle$ is a complete
lattice, referred to as the \emph{power set lattice}. Thus, most of
the work on valuation algebras assumes that the lattice $D$ is the
power set lattice of a set $V$ of variables. 

Let $V=\{x_{1},\ldots,x_{n}\}$ be a finite set of variables%
\footnote{All the definitions are correct not only for finite but also for countable
$V$%
}. We assume that for each variable $x\in V$ we can assign a set $\Omega_{x}$
of possible values, called its \emph{frame.} Similarly, the frame
of $X\subseteq V$ is $\Omega_{X}=\prod_{x\in X}\Omega_{x}.$ It is
mathematically convenient to include a singleton element (noted as
$\diamond$) in $\Omega_{\emptyset}.$ Thus, $\Omega_{\emptyset}=\{\diamond\}.$
A pair $\langle V,\Omega\rangle$ is known as a \emph{variable system.}
In many cases a variable system is naturally linked to a valuation
algebra. Then we say that the valuation algebra is equipped with a
variable system. A typical example is when valuations are discrete
real functions and the domain of a valuation is the set of discrete
variables over which the function is defined.

\begin{comment}
It is often needed to decompose tuples according to some variable
partition. We will call this operation \textit{projection }although
it is not directly related to valuation algebras operations .
\end{comment}
{} A \emph{tuple} $\mathbf{x}$ with finite domain $X\subseteq V$ is
an element of $\Omega_{X}.$ A projection operation can be defined
on tuples, unrelated to the projection operation of the valuation
algebra. Given a tuple $\mathbf{x}$ with domain $X$ and $Y\subseteq X$
we define the projection of $\mathbf{x}$ to $Y$ as the tuple $\mathbf{y}$
that results from $\mathbf{x}$ by discarding the values of the variables
in $X-Y.$ We note the projection of $\mathbf{x}$ to $Y$ as ${\bf x}^{\downarrow Y}.$
We can write $\mathbf{x}=(\mathbf{x}^{\downarrow Y},\mathbf{x}^{\downarrow X-Y})$.
Furthermore, $\mathbf{x}=(\mathbf{x},\diamond)=(\diamond,\mathbf{x})$.

\begin{example}
\label{ex:ValAlgBooleanFunctions}%
\begin{comment}
Next we will introduce the valuation algebra of Boolean functions,
which is used to solve dynamic programming with binary variables.
\end{comment}

Given a set of binary variables $V$, we consider its power set lattice
as the domain lattice of the valuation algebra. The set of valuations
is composed by all the functions $\phi:\Omega_{X}\rightarrow\{0,1\}$
, where $X\subseteq V$. The labeling operation is defined by $d(\phi)=X$.
The combination of two valuations $\phi,\psi,$ is the valuation \emph{$(\phi\otimes\psi)(\mathbf{x})=\phi(\mathbf{x}^{\downarrow d(\phi)})+\psi(\mathbf{x}^{\downarrow d(\psi)}),$}
whereas the projection of a valuation $\phi$ with $d(\phi)=X$ to
a domain $Y\subseteq X$ is the valuation \emph{$\phi^{\downarrow Y}(\mathbf{y})={\displaystyle \max{}_{\mathbf{z}\in\Omega_{X-Y}}\phi(\mathbf{y},\mathbf{z}).}$
}As proven in \cite{Pouly2011a} the valuation algebra of Boolean
functions satisfies axioms A1-A6.

Some other relevant examples of valuation algebras are relational
algebra, which is fundamental to databases, or the algebra of probability
potentials, which underlies many results in probabilistic graphical
models and the more abstract class of semiring induced valuation algebras
\cite{Kohlas2008}.
\end{example}

\subsection{Finding the marginal of a factorized valuation}

A relevant problem in many valuation algebras in the problem of finding
the marginal of a factorized valuation. 
\begin{problem}
\label{pbm:Marginalization}Let $(\Phi,D)$ be a valuation algebra,
and $\phi_{1},\ldots,\phi_{n}$ be valuations in $\Phi$. Find $\left(\phi_{1}\times\dots\times\phi_{n}\right)^{\downarrow X}.$
\end{problem}
Note that when our valuations are probability potentials, this is
the well studied problem of finding the marginal of a factorized distribution,
also known as Markov Random Field. 

\begin{comment}
Consider $\phi$ to be a big information package and $\phi_{1}\times\dots\times\phi_{n}$
one of its factorizations, i.e. . If we suppose that each factor $\phi_{i}$
is provided by an agent, we want each agent to know all the information
about $\phi$ that is relevant for him. Hence, we want to compute
$\phi^{\downarrow d(\phi_{i})}$ in a efficient way. Notice that due
to $d(\phi)$ size, we consider that $\phi_{1}\times\dots\times\phi_{n}$
is not efficient to compute. In particular we can not determine the
marginal of $\phi$ to $d(\phi_{i})$ by computing $(\phi_{1}\times\dots\times\phi_{n})^{\downarrow d(\phi_{i})}$.
\end{comment}

The \noun{Fusion} algorithm \cite{Shenoy1996} (a.k.a. variable elimination)
or the \noun{Collect} algorithm (a.k.a. junction tree or cluster tree
algorithm)\cite{Pouly2011a,Pouly2011c} can be used to find marginals.
Since our results build on top of the \noun{Collect} algorithm, we
provide a more accurate description below.

A necessary condition to apply the \noun{Collect} algorithm is that
we can organize the valuations $\phi_{1},\ldots,\phi_{n}$ into a
covering join tree, which we introduce next. 

\begin{comment}
These algorithms are base on message-sending between the elements
of the factorization. It is important to know when the messages are
sent and who the sender and receiver are. In order to understand why
the collect algorithm works we will introduce the concept of covering
join tree.
\end{comment}

\begin{defn}
A labeled tree is any tree $(V,E)$ together with a function $\lambda:V\rightarrow D$
that links each node with a single domain in $D$.

A join tree is a labeled tree $\mathcal{\mathcal{T}}=(V,E,\lambda,D)$
such that for any $i,j\in V$ it holds that $\lambda(i)\wedge\lambda(j)\leq\lambda(k)$
for all nodes $k$ on the path between $i$ and $j$. In that case,
we say that $\mathcal{T}$ satisfies the running intersection property.
\end{defn}
\begin{comment}
May the reader notice that the running intersection property allows
any two nodes to share messages about the common part in their domain
through the path between them.
\end{comment}

\begin{defn}
Given a valuation $\phi=\phi_{1}\times\dots\times\phi_{n}$ we say
that a join tree $\mathcal{\mathcal{T}}=(V,E,\lambda,D)$ is a covering
join tree for this factorization if $|V|=n$ and for all $\phi_{i}$
there is a node $j\in V$ such that $d(\phi_{i})\leq\lambda(j)$ .
\end{defn}
\begin{comment}
\begin{lyxgreyedout}
In order to simplify the reading we will number the nodes of a covering
join tree by using a map $\pi:V\rightarrow\mathbb{N}$ in the following
way:
\begin{itemize}
\item The root node, $v_{k}$, will be numbered as $\pi(v_{k})=r=|V|$.
\item Given two nodes $v_{i},v_{j}\in V\setminus\{v_{k}\}$ if $v_{j}$
is in the path from node $v_{i}$ to the root node, we will write
$\pi(v_{i})=a<b=\pi(v_{j})$.\end{itemize}
\end{lyxgreyedout}
\end{comment}

\begin{defn}
Let $i$ be a node in a rooted junction tree whose root is $r$. We
use $p(i)$ to denote the parent of $i$ in the tree. The separator
of $i$ is $s_{i}=\begin{cases}
\bot & \mbox{, if }i=r\\
\lambda(i)\wedge\lambda(p(i)) & \mbox{, otherwise}
\end{cases}$ 
\end{defn}

Algorithm \ref{alg:Collect} provides a description of the \noun{Collect}
algorithm. It is based on sending messages up the tree, through the
edges of the covering join tree, until the root node is reached. The
message sent from each node summarizes the information in the corresponding
subtree which is relevant to its parent. The running intersection
property guarantees that no information is lost.

\begin{algorithm}
{\footnotesize{Each node $i$ of the junction tree executes}}{\footnotesize \par}

{\footnotesize{\begin{algorithmic}[1]
\LineComment{Assess the product of valuations assigned to $i$}
\State $\psi_i \leftarrow e \times \prod_{j\in a^{-1}(i)}\phi_j$ 
\State From each child $j$ of $i$, receive a message $\mu_{j\rightarrow i}.$
\LineComment{Incorporate messages from children}
\State $\psi'_i \leftarrow  \psi_i\times \prod_j \mu_{j \rightarrow i}.$ 
\If {$i$ is not the root}
	%\State $\mu_{i \rightarrow p(i)} \leftarrow  
	\State Send message $\psi_i^{\downarrow s_i}$ to its parent $p(i)$
	%$\mu_{i\rightarrow p(i)}$ 
\EndIf
\end{algorithmic}}}{\footnotesize \par}

\caption{\label{alg:Collect}\noun{Collect} algorithm}
\end{algorithm}

\begin{thm}
After running the \noun{Collect} algorithm (Algorithm \ref{alg:Collect})
over the nodes of a covering join tree for $\phi=\prod_{k}\phi_{k}$,
we have that $\psi'_{i}=\left(\prod_{j\in\mathcal{T}_{i}}\psi_{j}\right)^{\downarrow\lambda(i)}.$
In particular, if $r$ is the root $\psi'_{r}=\phi^{\downarrow\lambda(r)}.$
\end{thm}
The theorem is an adaptation of Theorem 3.6 in \cite{Pouly2011a}.
As a consequence of this theorem, we can use the \noun{Collect} algorithm
to solve problem \ref{pbm:Marginalization} provided that we are given
a covering join tree for the factorization we would like to marginalize.

\begin{comment}
Es podrien demostrar els resultats 3.9 (que explica perquè es diu
``optimized'' covering join tree) i 3.10 del treball
\end{comment}

\section{Generic solutions in valuation algebras with variable system}

In the previous section we have shown that the \noun{Collect} algorithm
can be used to find marginals. In this section we focus on the \emph{solution
finding problem} (SFP). 

The problem is of foremost importance, since it lies at the foundation
of dynamic programming \cite{Shenoy1996,Bertele1972}. Furthermore,
problems such as satisfiability, solving Maximum a Posteriori queries
in a probabilistic graphical models, or maximum likelihood decoding
are particular instances of the SFP. 

We start by formally defining the problem and then we review the results
of Pouly \cite{Pouly2011c} and Pouly and Kohlas \cite{Pouly2011a},
who give algorithms for solving those problems and establish the conditions
under which those algorithms are guaranteed to work. Unfortunately,
although the inspirational ideas underlying Pouly and Kohlas' work
are correct, their formal development is not. Thus, we end up the
section providing a counter example to one of their main theorems.

\subsection{The SFP for valuation algebras with variable system}

Up to know, the most general formalization of the SFP is the one provided
by \cite{Shenoy1996} and adapted by Pouly and Kohlas to the formal
framework of valuation algebras in Chapter 8 of \cite{Pouly2011a}.
They assume a valuation algebra $\langle\Phi,D\rangle$ equipped with
a variable system $\langle V,\Omega\rangle$. As in the marginal assessment
problem, in the SFP, we are given a set of valuations $\phi_{1},\ldots,\phi_{m}\in\Phi$
as input. However, instead of a marginal of its combination $\phi=\phi_{1}\times\ldots\times\phi_{m}$,
we are required to provide a tuple $\mathbf{x}$ with domain $d(\phi)$,
such that $\mathbf{x}$ is a solution for $\phi$. In order for the
previous sentence to make any sense we need to properly define our
concept of \emph{solution. }The most general way in which we can do
this is by defining a family $c=\{c_{\phi}|\phi\in\Phi\}$ of solution
sets, such that for each valuation $\phi\in\Phi$, the solution set
$c_{\phi}$ contains the subset of $\Omega_{d(\phi)}$ such that $\mathbf{x}\in\Omega_{d(\phi)}$
is considered a solution for $\phi$ if and only if $\mathbf{x}\in c_{\phi}.$
We say that the family of sets $c$ is a \emph{solution concept}.
Now we can formally define the SFP as follows
\begin{problem}[SFP with variable system]
\noindent  Given a valuation algebra $\langle\Phi,D\rangle$ equipped
with a variable system $\langle V,\Omega\rangle$ and a solution concept
$c$, and a set of valuations $\phi_{1},\ldots,\phi_{m}\in\Phi,$
find $\mathbf{x}\in\Omega_{d(\phi)}$ such that $\mathbf{x}$ is a
solution for $\phi=\phi_{1}\times\ldots\times\phi_{m}.$ 
\end{problem}

\subsection{Solving the solution finding problem by composing partial solutions}

Several authors have provided algorithms that solve the SFP and characterized
under which conditions they can be successfully applied.

As described in the introduction, several works have sought to provide
a formal foundation to dynamic programming, which we can now identify
as a particular case of the SFP. In their works in 2011, Pouly and
Kohlas \cite{Pouly2011a,Pouly2011c} drop the assumption that valuations
are functions that map tuples into a value set $\Delta.$ They present
several algorithms, and characterize the sufficient conditions for
its correctness. By dropping the assumption that valuations are functions,
their algorithms can be applied to previously uncovered cases . 

Essentially, the sufficient conditions for the correctness of Pouly
and Kohlas' algorithms connect the operations in the valuation algebra
with the solution concept by means of a \emph{family of configuration
extension set}s.%
\footnote{Although they do never formally introduce families of configuration
extension sets, we have introduced the concept here for mathematical
correctness and so that the reader can easily follow the generalization
that will come later on.%
}A family of configuration extension sets $\mbox{\ensuremath{\mathcal{W}}}$
assigns a configuration extension set to each pair $\langle\phi,\mathbf{x}\rangle$
such that $\phi$ is a valuation and $\mathbf{x}$ is a tuple whose
domain $X$ is a subset of $d(\phi).$ That is $\mathcal{W}=\{W_{\phi}(\mathbf{x})\subseteq\Omega_{d(\phi)-X}|\phi\in\Phi,X\subseteq d(\phi),\mathbf{x}\in\Omega_{X}\}.$
Furthermore, a family of configuration extension sets has to satisfy
two conditions. The first one connects the projection operation of
the valuation algebra with the extension sets by imposing that every
extension set can be calculated in two steps. The second one connects
the set of solutions of a valuation with the set of extensions of
the empty tuple $\diamond$. These conditions can be stated formally
as follows
\begin{enumerate}
\item For each $\phi\in\Phi$, for each $X\subseteq Y\subseteq d(\phi)$
and for each $\mathbf{x}\in\Omega_{X}$ we have that $W_{\phi}(\mathbf{x})=\{\mathbf{z}\in\Omega_{d(\phi)-X}\mid\mathbf{z}^{\downarrow Y-X}\in W_{\phi^{\downarrow Y}}(\mathbf{x})\mbox{ and }\mathbf{z}^{\downarrow d(\phi)-Y}\in W_{\phi}(\mathbf{x},\mathbf{z}^{\downarrow Y-X})\}.$
\item For each $\phi\in\Phi$, $c_{\phi}=W_{\phi}(\diamond).$
\end{enumerate}
Based on this definition, Pouly and Kohlas state the following theorem 
\begin{thm}[Theorem 8.1 in \cite{Pouly2011a}]
\label{thm:8.1} For any valuation $\phi\in\Phi$ and any $X,Y\subseteq d(\phi)$,
we have {\footnotesize{
\begin{equation}
c_{\phi}^{\downarrow X\cup Y}=\{\mathbf{z}\in\Omega_{X\cup Y}\mid\mathbf{z}^{\downarrow Y}\in c_{\phi}^{\downarrow Y}\mbox{ and }\mathbf{z}^{\downarrow X-Y}\in W_{\phi^{\downarrow X}}(\mathbf{z}^{\downarrow X\cap Y})\}.\label{eq:thm:8.1}
\end{equation}
}}{\footnotesize \par}
\end{thm}

Unfortunately, the theorem is not correct. We will use the valuation
algebra on Boolean functions from example \ref{ex:ValAlgBooleanFunctions}to
build a counterexample for it. Pouly and Kohlas (equation 8.19) defined
the extension sets for the algebra of Boolean lattices as $W_{\phi}(\mathbf{x})=\{\mathbf{y}\in\Omega_{d(\phi)-X}\mid\phi(\mathbf{x},\mathbf{y})=\phi^{\downarrow X}(\mathbf{x})\}$,
where $X\subseteq d(\phi)$, $\mathbf{x}\in\Omega_{X}$. 

\begin{counter}\label{cnt:Boolean}

Taking $\phi$ as the Boolean function $\phi(\mathbf{x},\mathbf{y})=\begin{cases}
1 & \mbox{if \ensuremath{\mathbf{x}=\mathbf{y}}},\\
0 & \mbox{otherwise}
\end{cases}$, $X=\{x\}$, and $Y=\{y\}$ the result in theorem \ref{thm:8.1}
does not hold.

\end{counter}

To see why, we can assess both sides of equation \ref{eq:thm:8.1}
and see that they are not the same. For the left hand side, note that
$c_{\phi}^{\downarrow X\cup Y}=c_{\phi}=W_{\phi}(\diamond).$ Thus,
$c_{\phi}=\{(\mathbf{x},\mathbf{y})\in\Omega_{X\cup Y}\mid\phi(\mathbf{x},\mathbf{y})=\phi^{\downarrow\emptyset}(\diamond)\}$.
We can assess $\phi^{\downarrow\emptyset}(\diamond)=\max_{x.y}\phi(\diamond,(\mathbf{x},\mathbf{y}))=\max_{x,y}\phi(\mathbf{x},\mathbf{y})=1$.
Hence, $c_{\phi}^{\downarrow X\cup Y}=c_{\phi}=\{(0,0),(1,1)\}$.

Let $A$ denote the set at the r.h.s. of equation \ref{eq:thm:8.1}.
Since $X\cap Y=\emptyset$ and $X-Y=X,$ we have that $A=\{\mathbf{z}\in\Omega_{X\cup Y}\mid\mathbf{z}^{\downarrow Y}\in c_{\phi}^{\downarrow Y}\mbox{ and }\mathbf{z}^{\downarrow X-Y}\in W_{\phi^{\downarrow X}}(\mathbf{z}^{\downarrow X\cap Y})\}=\{\mathbf{z}\in\Omega_{X\cup Y}\mid\mathbf{z}^{\downarrow Y}\in c_{\phi}^{\downarrow Y}\mbox{ and }\mathbf{z}^{\downarrow X}\in W_{\phi^{\downarrow X}}(\diamond)\}$.
We can assess $c_{\phi}^{\downarrow Y}=\{\mathbf{z}^{\downarrow Y}\mid\mathbf{z}\in c_{\phi}\}=\{(0),(1)\}$.
Furthermore, we have that $W_{\phi^{\downarrow X}}(\diamond)=\{\mathbf{x}\in\Omega_{X}|\mid\phi^{\downarrow X}(\mathbf{x})=(\phi^{\downarrow X})^{\downarrow\emptyset}(\diamond)\}=\{\mathbf{x}\in\Omega_{X}\mid\phi^{\downarrow X}(\mathbf{x})=\phi{}^{\downarrow\emptyset}(\diamond)\}=\{\mathbf{x}\in\Omega_{X}\mid\phi^{\downarrow X}(\mathbf{x})=1\}$.
We have that for all $\mathbf{x}\in X,$ $\phi^{\downarrow X}(\mathbf{x})=\max_{\mathbf{y}}\phi(\mathbf{x},\mathbf{y})=1$
, and thus, $W_{\phi^{\downarrow X}}(\diamond)=\Omega_{X}=\{0,1\}$.
Hence, $A=\{\mathbf{z}\in\Omega_{X\cup Y}\mid\mathbf{z}^{\downarrow Y}\in\Omega_{Y}\mbox{ and }\mathbf{z}^{\downarrow X}\in\Omega_{X}\}=\Omega_{X\cup Y}=\{(0,0),(0,1),(1,0),(1,1)\}\}\neq c_{\phi}^{X\cup Y}$,
contradicting equation \ref{eq:thm:8.1}.

Summarizing, in their works \cite{Pouly2011c,Pouly2011a} in 2011,
Pouly and Kohlas make an attempt to generalize the results of Shenoy
to valuation algebras equipped with a variable system, not restricting
the valuations to be functions into a value set $\Delta.$ However,
as proved by counterexample \ref{cnt:Boolean}, one of the key results
in their development is not correct. Since the correctness proofs
provided by Pouly and Kohlas for their algorithms rely on this result,
what could be a minor technical detail ends up having strong consequences
for the validity of the theory as a whole. 

The main objective of the next section is to identify necessary conditions
for the application of the algorithms presented by Pouly and Kohlas
and to prove that their correctness under those conditions.

\begin{comment}
\input{NewProblems.tex}
\end{comment}

\section{Even more generic solutions in valuation algebras}

During our efforts to identify the necessary conditions for the application
of the algorithms presented by Pouly and Kohlas we realized that nothing
in the theory we were building required that the valuation algebra
was equipped with a variable system. Thus, as a byproduct of the correction
effort, the resulting theory is the first one that proposes a generic
algorithm, the so-called\noun{ Collect+Extend} algorithm, to solve
the SFP for valuation algebras which are not necessarily equipped
with a variable system. The generality of the \noun{Collect+Extend}
algorithm allows it to be applied to valuation algebras such as the
algebra of sparse potentials, an example that until now was not covered
by any previous formalization.

We start this section by generalizing the definition of the SFP problem
so that it does not enforces the valuation algebra to be equipped
with a variable system. Then, we introduce the concept of piecewise
extensibility, and we prove that it is a sufficient condition for
the correctness of the \noun{Collect+Extend.} Finally, we introduce
the \noun{Collect+ExtendAll }algorithm, whose objective is obtaining
not a single solution to the SFP, but every solution. We introduced
fully piecewise extensibility and prove that it is a sufficient condition
for the correctness of the \noun{Collect+ExtendAll }algorithm.

\subsection{A more general solution finding problem}

We start by introducing the concept of configuration system, a generalization
of the concept of variable system that does not enforce tuples to
be members of a Cartesian product. Then we generalize the SFP to configuration
systems.

\paragraph{Configuration systems, compatibility and merge-friendliness.}

We start by relating each element of the domain lattice with a set
of configurations, and then we impose a minimal constraint among those
sets of configurations, resulting in the notion of configuration system. 
\begin{defn}[Configuration system]
Given a lattice $D,$ a configuration system $\langle\Gamma,\pi\rangle$
is composed of (i) a set of configurations $\Gamma_{s}$ for each
$s\in D$ and (ii) for each pair of domains $s,t\in D$ such that
$s\leq t,$ a surjective mapping $\pi_{t\rightarrow s}:\Gamma_{t}\rightarrow\Gamma_{s}.$
Without loss of generality, the configuration sets are assumed to
be mutually exclusive. Furthermore, $\Gamma_{\bot}=\{\lozenge\}.$ 

Whenever $\mathbf{x}\in\Gamma_{t},$ we say that $\mathbf{x}$ is
a configuration with scope $t.$ Given $\mathbf{x}\in\Gamma_{t},$
we note $\mathbf{x}_{s}=\pi_{t\rightarrow s}(\mathbf{x}).$ 
\end{defn}
It is easy to see that any variable system is a configuration system.
However, there are configuration systems which do not have an equivalent
variable system. 

A relevant concept in a configuration system is that of compatibility
between configurations. 
\begin{defn}[Compatibility, merger, merge-friendly]
Let $s,t\in D$ and let $\mathbf{x}\in\Gamma_{s},$ and $\mathbf{y}\in\Gamma_{t}.$
We say that $\mathbf{x}$ and $\mathbf{y}$ are \emph{compatible}
whenever there is $\mathbf{z}\in\Gamma_{s\vee t}$ such that $\mathbf{z}_{s}=\mathbf{x}$
and $\mathbf{z}_{t}=\mathbf{y}.$ We say that such a $\mathbf{z}$
is a \emph{merger} of $\mathbf{x}$ and $\mathbf{y.}$ The definitions
of compatibility and merger can be easily extended to a set of configurations
instead of two. A configuration system is \emph{merge-friendly} if
for any $s,t\in D,$ any $\mathbf{x}\in\Gamma_{s},$ and any $\mathbf{y}\in\Gamma_{t}$
whenever $\mathbf{x}_{s\wedge t}=\mathbf{y}_{s\wedge t}$, we have
that $\mathbf{x}$ and $\mathbf{y}$ are compatible.
\end{defn}
We are interested in merge-friendly configuration systems where a
merger of a set of compatible configurations can be efficiently found.
For example, in variable systems we can understand each tuple as restricting
the values of some variables. Two tuples are compatible when there
is no variable to which they assign a different value, and a merger
can be easily obtained by imposing simultaneously the restrictions
of both tuples.

\paragraph{The solution finding problem for valuation algebras with configuration
systems.}

The definition of solution concept can be migrated from variable system
to configuration system. In the latter case, a configuration system
is a family $c=\{c_{\phi}|\phi\in\Phi\}$ of solution sets, such that
for each valuation $\phi\in\Phi$, the solution set $c_{\phi}$ contains
the subset of $\Gamma_{d(\phi)}$ such that $\mathbf{x}\in\Gamma_{d(\phi)}$
is considered a solution for $\phi$ if and only if $\mathbf{x}\in c_{\phi}.$
The generalization of the solution finding problem is as follows
\begin{problem}[SFP]
\noindent  Given a valuation algebra $\langle\Phi,D\rangle$ equipped
with a configuration system $\langle\Gamma,\pi\rangle$, and a solution
concept $c$, and a set of valuations $\phi_{1},\ldots,\phi_{m}\in\Phi,$
find $\mathbf{x}\in\Gamma_{d(\phi)}$ such that $\mathbf{x}$ is a
solution for $\phi=\phi_{1}\times\ldots\times\phi_{m}.$ 
\end{problem}
Following Pouly and Kohlas, in order to be able to state the algorithms
that solve the SFP we need the solution concept to lie inside a family
of configuration extension sets. This connects the configuration system
of the domain lattice with the marginalization operation of the valuation
algebra and with the solution concept.

\begin{defn}[Family of configuration extension sets]
Given a valuation algebra $\langle\Phi,D\rangle,$ and a configuration
system $\langle\Gamma,\pi\rangle$ over $D,$ and a solution concept
$c,$ a \emph{family of configuration extension sets }is a family
of sets $\mathcal{E}=\{E_{\phi}(\mathbf{x})|d\in D,\phi\in\Phi,\mathbf{x}\in\Gamma_{d}\}$,
that 
\begin{enumerate}
\item For all $\phi\in\Phi$, for all $\mathbf{x}\in\mbox{\ensuremath{\Gamma}}_{d(\phi)}$
, 
\begin{equation}
E_{\phi}(\mathbf{x})=\{\mathbf{x}\}.\label{eq:I-Am-My-Extension}
\end{equation}

\item For all $\phi\in\Phi$, for all $s,t\in D$ such that $s\neq t$ and
$s\leq t\leq d(\phi)$, and for all $\mathbf{x}\in\Gamma_{s}$ , 
\begin{equation}
E_{\phi}(x)=\{\mathbf{y}\in\Gamma{}_{d(\phi)}|\mathbf{y}_{t}\in E_{\phi^{\downarrow t}}(\mathbf{x})\mbox{ and }\mathbf{y}\in E_{\phi}(\mathbf{y}_{t})\}.\label{eq:Extension-By-Parts}
\end{equation}

\item For all $\phi\in\Phi$, 
\begin{equation}
c_{\phi}=E_{\phi}(\diamond).\label{eq:Solution-Is-Diamond-Extension}
\end{equation}

\end{enumerate}
Whenever $\mathbf{y}\in E_{\phi}(x)$ we say that $\mathbf{y}$ is
an extension of $\mathbf{x}$ to $\phi.$ Note that in order for $\mathbf{y}$
to be an extension of $\mathbf{x},$ the scope $s$ of $\mathbf{x}$
must be smaller than the scope of $\mathbf{y}.$ We can extend the
definition of extension to a scope $u$ whatsoever: given $\mathbf{x}\in\Gamma_{u}$
we say that $\mathbf{y}\in\Gamma_{t}$ is an extension of $\mathbf{x}$
to $\phi$ if $\mathbf{y}\in E_{\phi}(\mathbf{x}_{t\wedge u})$. This
states that $\mathbf{y}$ is an extension of $\mathbf{x}$ if it $\mathbf{y}$
is an extension of that part of $\mathbf{x}$ which is of interest
to $\phi.$

Next, we introduce the\noun{ Collect+Extend} algorithm. The algorithm
can be run on any valuation algebra equipped with a configuration
system and a family of configuration extension sets. As the \noun{Collect}
algorithm, the\noun{ Collect+Extend} algorithm requires the existence
of a covering join tree for the factorization. It has two different
phases. During the first phase, the \noun{Collect} algorithm is used
to obtain the marginal of $\phi$ at the root of the tree. After that,
during the second phase, the root starts from an empty configuration
($\diamond$), and selects a configuration that belongs to the set
of extensions of $\diamond$ to his marginal. From there on, each
node $i$ of the tree receives from his parent $p(i)$ enough information
from the configuration selected by $p(i)$ so that $i$ can successfully
extend it to his domain configuration set. We call this second phase
the \noun{Extend} phase. Algorithm \ref{alg:Extend} describes it
in a more precise way. At the end of the \noun{Extend} phase, each
node $i$ of the tree has a configuration $\eta_{i}$ over its domain.
Provided that these configurations are compatible, we get a single
configuration in $\Gamma_{d(\phi)}$ by assessing its merger.
\end{defn}
\begin{algorithm}
{\footnotesize{Each node $i$ of the junction tree executes}}{\footnotesize \par}

{\footnotesize{\begin{algorithmic}[1]
\If {$i$ is the root}
	\State $\nu_i \leftarrow \diamond$
\Else
	\State From its parent $pa(i)$, receive a message $\nu_i.$
\EndIf
\LineComment{Extend the parent solution to $i$'s scope.}
\State Select $\eta_i \in  E_{\psi'_i}(\nu_i).$
\ForAll{$j$ children of $i$}
	\State Send message %nu_{i \rightarrow j} \leftarrow 
			$\nu_j=\pi_{d_i\rightarrow s(j)}(\eta_i)$ to children $j$.
\EndFor
\end{algorithmic}}}{\footnotesize \par}

\caption{\label{alg:Extend}\noun{Extend }algorithm}
\end{algorithm}

\subsection{Sufficient conditions for the correctness of the \noun{Collect+Extend}
algorithm}

In this section we consider the problem of determining under which
conditions the configuration assessed by the \noun{Collect+Extend}
algorithm is a solution to the SFP problem. The main result is the
following theorem

\begin{restatable}[\noun{Collect+Extend} suff. cond.]{thm}{CollectExtendThm} \label{thm:Collect+Extend}
Let $\langle \Phi,D\rangle$ be a valuation algebra equipped with a {\bf merge-friendly} configuration system $\langle \Gamma,\phi\rangle,$ and a solution concept $c.$ Let $\mathcal{E}$ be a {\bf piecewise extensible} family of configuration extension sets. After running the \noun{Collect} algorithm followed by the \noun{Extend} algorithm (Algorithms \ref{alg:Collect} and \ref{alg:Extend}) over the nodes of a covering join tree $\mathcal{T}$ for $\phi=\prod_{k=1}^{n}\phi_{k}$, we have that there is at least a merger $\mathbf{z}$ of $\{\eta_{i}|i \in V\}$, and that $\mathbf{z}$ is a solution of $\phi.$  
\end{restatable}

The theorem requires the family of configuration extension sets $\mathcal{E}$
to be piecewise extensible. Next, we define piecewise extensibility
and then we prove that it is a sufficient condition for the correctness
of the \noun{Collect+Extend} algorithm.

\paragraph{Piecewise extensibility}

Intuitively, this means requiring that whenever a configuration \textbf{$\mathbf{z}$}
independently belongs to the set of extensions of two different valuations
$\phi_{1}$ and $\phi_{2}$, then it does belong to the set of extensions
to its product. 
\begin{defn}
\label{def:Piecewise-Extensible}A configuration $\mathbf{x}$ (with
scope $s$) is \emph{extensible} to a valuation $\phi$ (with domain
$t$) whenever $E_{\phi}(\mathbf{x}_{s\wedge t})\neq\emptyset.$ For
any $\mathbf{z}\in E_{\phi}(\mathbf{x}_{s\wedge t}),$ we say that
$\mathbf{z}$ is an extension of $\mathbf{x}$ to $\phi.$ 

A\emph{ }family of configuration extension sets\emph{ }$\mathcal{E}$\emph{
}is \emph{piecewise extensible }when for any two valuations $\phi_{1},\phi_{2}\in\Phi$
with $d_{1}=d(\phi_{1})$ and $d_{2}=d(\phi_{2}),$ any $t\in D$,
$d_{1}\vee d_{2}\geq t\geq d_{1}\wedge d_{2},$ any $\mathbf{x}\in\Gamma_{t}$
and any extension $\mathbf{z}$ of $\mathbf{x}$ to both $\phi_{1}$
and $\phi_{2}$ , we have that $\mathbf{z}$ is an extension of $\mathbf{x}$
to $\phi_{1}\times\phi_{2}$.

\end{defn}

Note that the piecewise extensibility requirement is defined only
for pairs of valuations $\phi_{1}$ and $\phi_{2}$. The following
lemma shows that provided that we have piecewise extensibility for
two valuations, we can extend it to products of $m$ valuations.
\begin{lem}
\label{lem:Piecewise-Extension-N-ary}Let (i) $\mathcal{E}$ be a
piecewise extensible family of configuration extension sets, (ii)
\textup{$\phi_{\text{1}},\ldots,\phi_{m}\in\Phi,$}\textup{\emph{
and $\phi=\prod_{i=1}^{m}\phi_{i},$ and (iii) $t\in D,$ such that
$\bigvee_{i=1}^{m}d_{i}\geq t\geq\bigvee_{i=1}^{m}r_{i},$ where $r_{i}=d_{i}\text{\ensuremath{\wedge\left(\ensuremath{\bigvee_{j\neq i}d_{j}}\right)}},$
and $d_{i}=d(\phi_{i}).$}}\textup{ }For any $\mathbf{x}\in\Gamma_{t}$
and any extension $\mathbf{z}$ of $\mathbf{x}$ to $\phi_{1},\ldots,\phi_{m}$
we have that\textup{ }$\mathbf{z}$ is an extension of $\mathbf{x}$
to $\phi.$\end{lem}
\begin{proof}
By induction on the number of terms $m.$ 

If $m=2$, the result follows directly from the definition of piecewise
extensible. 

Assume it is true for $m<M,$ and prove it for $m=M.$ We have that
$\phi=\prod_{i=1}^{M}\phi_{i}.$ We can break it as $\phi=\phi_{1}\times\prod_{i=2}^{M}\phi_{i}=\phi_{1}\times\xi,$
making $\xi=\prod_{i=2}^{M}\phi_{i}.$ 

We know $\mathbf{z}$ is an extension of $\mathbf{x}$ to $\phi_{1}.$
We can apply the induction hypothesis to $\xi=\prod_{i=2}^{M}\phi_{i},$
to show that $\mathbf{z}$ is also an extension of $\mathbf{x}$ to
$\xi$ and then we apply piecewise extensibility to conclude that
$\mathbf{z}$ is an extension of $\mathbf{x}$ to $\phi.$ 

In order to apply the induction hypothesis, we take $\xi=\prod_{i=2}^{M}\phi_{i},$
and $t'=t\wedge d(\xi).$ It is easy to see that $\bigvee_{i=2}^{M}d_{i}=d(\xi)\geq t\wedge d(\xi)=t'.$
On the other hand, $t'\geq\bigvee_{i=2}^{M}r_{i}.$ Since $\mathbf{z}$
is an extension of $\mathbf{x}_{t'}$ to $\phi_{2},\ldots,\phi_{M}$,
we have that for each $i,$ $\mathbf{z}\in E_{\phi_{i}}(\mathbf{x}{}_{s{}_{i}}),$
where $s_{i}=t\wedge d_{i}.$ 

Note that $d(\xi)=\bigvee_{i=2}^{M}d_{i}$. Take $t'=t\wedge d(\xi).$
We can see that $\bigvee_{i=2}^{M}d_{i}\geq t'\geq\bigvee_{i=2}^{M}r_{i}.$
Furthermore, for $i\neq1,$ we have that $s'_{i}=t'\wedge d_{i}=t\wedge d(\xi)\wedge d_{i}=t\wedge(\bigvee_{i=2}^{M}d_{i})\wedge d_{i}=t\wedge d_{i}=s_{i}$,
and hence $\mathbf{z}$ is an extension of $ $$\mathbf{z}\in E_{\phi_{i}}(\mathbf{x}{}_{s'_{i}}).$
Applying the induction hypothesis we have that $\mathbf{z}_{d(\xi)}\in E_{\xi}(\mathbf{x}_{t'})=E_{\xi}(\mathbf{x}_{t\wedge d(\xi)}).$ 

The conditions to apply piecewise extensibility to $\phi_{1}\times\xi$
are now in place. Observe that $d_{1}\vee d(\xi)\geq t\geq\left(\bigvee_{i=2}^{M}d_{i}\right)\wedge d_{1}=d(\xi)\wedge d_{1}$,
that $\mathbf{z}_{d_{1}}\in E_{\phi_{1}}(\mathbf{x}{}_{s_{1}}),$
and that $\mathbf{z}_{d(\xi)}\in E_{\xi}(\mathbf{x}{}_{t\wedge d(\xi)}).$
Since $\mathcal{E}$ is piecewise extensible we have that $\mathbf{z}$
is an extension of $\mathbf{x}$ to $\phi.$
\end{proof}

\paragraph{Sufficient condition for the correctness of \noun{Collect+Extend.}}

Next, we see that on piecewise extensible family of configuration
extension sets, it is possible to take benefit of the factorization
of the valuation to find a solution by merging partial solutions to
the different factors which are coherent between them. We start by
proving this for a product of two valuations and a product of $m$
valuations. Then we apply those results to prove that the \noun{Collect+Extend
}algorithm is correct. 

We start proving the following lemma, that shows that provided we
have piecewise extensibility, for any valuation that is the product
of two factors, if we are given a solution to the projection of the
product to the domain of one of the factors and an extension of that
solution to the second factor, the merger of these two is a solution
to the product.
\begin{lem}
\label{lem:Solution-Extension-Binary}Let $\mathcal{E}$ be a piecewise
extensible family of configuration extension sets. Let $\phi_{1},\phi_{2}\in\Phi,$
and let $d_{1}=d(\phi_{1})$, $d_{2}=d(\phi_{2}),$ and $\phi=\phi_{1}\times\phi_{2}.$\textup{
}\textup{\emph{For any $\mathbf{x}_{\mathbf{}}\in c_{\phi^{\downarrow d_{1}}},$any
extension $\mathbf{y}$ of $\mathbf{x}$ to $\phi_{2}$$,$ and any
merger $\mathbf{z}$ of $\mathbf{x}$ and $\mathbf{y}$, we have that
$\mathbf{z}\in c_{\phi}.$}}\end{lem}
\begin{proof}
First, we will use piecewise extensibility to prove that $\mathbf{z}\in E_{\phi}(\mathbf{x}).$
To do that we apply definition \ref{def:Piecewise-Extensible} with
$t=d_{1}.$ By hypothesis, we have that $\mathbf{z}_{d_{2}}=\mathbf{y}\in E_{\phi_{2}}(\mathbf{x}_{d_{2}\wedge d_{1}}).$
From equation \ref{eq:I-Am-My-Extension}, we have that $E_{\phi_{1}}(\mathbf{x})=\{\mathbf{x}\}.$
Hence, $\mathbf{z}_{d_{1}}=\mathbf{x}\in E_{\phi_{1}}(\mathbf{x}).$
Thus, $\mathbf{z}$ is a coherent extension of $\mathbf{x}$ to both
$\phi_{1}$ and $\phi_{2}$ and we can apply piecewise extensibility
to conclude that $\mathbf{z}\in E_{\phi}(\mathbf{x}).$ 

Then, we can jointly apply equations \ref{eq:Solution-Is-Diamond-Extension}
and \ref{eq:Extension-By-Parts} to conclude that $\mathbf{z}\in c_{\phi}=E_{\phi}(\diamond).$
\end{proof}
As in the previous section, we can generalize this result to products
of $m$ valuations.
\begin{lem}
\label{lem:Solution-Extension-N-ary}Let $\mathcal{E}$ be a piecewise
extensible family of configuration extension sets. Let \textup{$\phi=\phi_{\rho}\times\prod_{i=1}^{m}\phi_{i},$
with $d_{i}=d(\phi_{i}),$ $d_{\rho}=d(\phi_{\rho}),$ $r_{i}=d_{i}\text{\ensuremath{\wedge\left(\ensuremath{\bigvee_{j\neq i}d_{j}}\right)}},$
and $d_{\rho}\geq\ensuremath{\bigvee_{i=1}^{m}r_{i}}.$ Given $\mathbf{x}\in c_{\phi^{\downarrow d_{\rho}}},$
for each $1\leq i\leq m,$ $\mathbf{y}_{i}$ an extension of $\mathbf{x}$
to $\phi_{i}$, and any merger $\mathbf{z}$ of $\mathbf{x}$ and
$\mathbf{y}_{1},\ldots,\mathbf{y}_{m},$ we have that $\mathbf{z}\in c_{\phi}.$ }\end{lem}
\begin{proof}
Let By induction on $m.$ 

If $m=1,$ we can directly apply Lemma \ref{lem:Solution-Extension-Binary}. 

Assume it is true for $m<M.$ We have to prove that it is true for
$m=M.$ Let $\xi=\prod_{i=1}^{m}\phi_{i}.$ First, we will prove that
$\mathbf{z}_{d_{\xi}}\in E_{\xi}(\mathbf{x}_{d_{\rho}\wedge d_{\xi}})$
and then we will apply Lemma \ref{lem:Solution-Extension-Binary}.
To see that $\mathbf{z}_{d_{\xi}}\in E_{\xi}(\mathbf{x}_{d_{\rho}\wedge d_{\xi}})$,
we apply Lemma \ref{lem:Piecewise-Extension-N-ary} with $t=d_{\rho}\wedge d_{\xi}.$ 

We need to verify that $\bigvee_{i=1}^{m}d_{i}\geq t\geq\bigvee_{i=1}^{m}r_{i}.$
The left inequality is satisfied since $t=d_{\rho}\wedge d_{\xi}=d_{\rho}\wedge\left(\bigvee_{i=1}^{m}d_{i}\right)\leq\bigvee_{i=1}^{m}d_{i}.$
Since for all $i,$ $r_{i}\leq d_{\xi},$ we have that $\bigvee_{i=1}^{m}r_{i}\leq d_{\xi}.$
Since by hypothesis we have that $d_{\rho}\geq\bigvee_{i=1}^{m}r_{i},$
we can conclude that $d_{\rho}\wedge d_{\xi}\geq\bigvee_{i=1}^{m}r_{i}.$ 

Furthermore we can verify that for each $i\in\{1,\ldots,m\},$ $\mathbf{z}_{d_{i}}=\mathbf{y}_{i}\in E_{\phi_{i}}(\mathbf{x}{}_{t\wedge d_{i}}).$
Note that $t\wedge d_{i}=d_{\rho}\wedge d_{\xi}\wedge d_{i}=d_{\rho}\wedge\left(\bigvee_{i=1}^{m}d_{i}\right)\wedge d_{i}=d_{\rho}\wedge d_{i},$
and by hypothesis we have that $\mathbf{z}_{d_{i}}=\mathbf{y}_{i}\in E_{\phi_{i}}(\mathbf{x}_{d_{\rho}\wedge d_{i},}).$

Applying Lemma \ref{lem:Piecewise-Extension-N-ary}, we get that $\mathbf{z}_{d_{\xi}}\in E_{\xi}(\mathbf{x}_{d_{\rho}\wedge d_{\xi}})$
and since by hypothesis we have that $\mathbf{x}\in c_{\phi^{\downarrow d_{\rho}}},$
we can conclude from Lemma \ref{lem:Solution-Extension-Binary} that
$\mathbf{z}\in c_{\phi_{\rho}\times\phi_{\xi}}=c_{\phi}$. 
\end{proof}
The former results allow us to state the following theorem (the main
result of the section) proving that when a valuation breaks as a product
of smaller valuations, the \noun{Collect+Extend} algorithm can be
used to assess a solution to it. 

\CollectExtendThm*

\begin{proof}
By induction on the number of nodes of the junction tree. If the junction
tree has only one node then the proof is trivial. 

Assume that the junction tree has $m>1$ nodes. We can see that the
conditions to apply Lemma \ref{lem:Solution-Extension-N-ary} are
satisfied at the root. 
\begin{enumerate}
\item $\phi=\psi_{\rho}\times\prod_{i\in Children(\rho)}\xi_{i},$ where
$\xi_{i}=\prod_{j\in\mathcal{T}_{i}}\psi_{i}.$ Let $d'_{i}=d(\xi_{i})$
\item For each $i,j\in Children(\rho)$ such that $i\neq j,$ due to the
running intersection property we have that $d_{\rho}\geq d'_{i}\wedge\left(\bigvee_{j\neq i}d'_{j}\right).$ 
\item $\mathbf{x}\in c_{\phi^{\downarrow d_{\rho}}}.$ 
\item For each $i\in Children(\rho)$, $\mathbf{y}_{i}\in E_{\xi_{i}}(\mathbf{x}{}_{d_{\rho}\wedge d'_{i}})$
\item $\mathbf{z}$ is a merger of $\mathbf{x}$ and each of the $\mathbf{y}_{i}'s$
\end{enumerate}
As a consequence of lemma \ref{lem:Solution-Extension-N-ary}, we
can conclude that $\mathbf{z}\in c_{\phi}.$
\end{proof}

\paragraph{Fully piecewise extensibility and assessing all solutions.}

By strengthening the concept of piecewise extensibility, we can use
an algorithm similar to \noun{Collect+Extend} to assess all solutions
instead of only one. The strengthening is named fully piecewise extensibility.
\begin{defn}
\label{def:Fully-Piecewise-Extensible}A\emph{ }family of configuration
extension sets\emph{ }$\mathcal{E}$\emph{ }is \emph{fully piecewise
extensible }when for any two valuations $\phi_{1},\phi_{2}\in\Phi$
with $d_{1}=d(\phi_{1})$ and $d_{2}=d(\phi_{2}),$ any $t\in D$,
$d_{1}\vee d_{2}\geq t\geq d_{1}\wedge d_{2},$ any $\mathbf{x}\in\Gamma_{t}$
and any $\mathbf{z},$ we have that $\mathbf{z}$ is an extension
of $\mathbf{x}$ to both $\phi_{1}$ and $\phi_{2}$ , \textbf{if
and only if} $\mathbf{z}$ is an extension of $\mathbf{x}$ to $\phi_{1}\times\phi_{2}$.

The algorithm and the theorem that shows that this is a sufficient
condition are omitted due to lack of space but can be derived without
effort.\end{defn}

%
\begin{comment}
\input{SparsePotentialsAlgebra.tex}
\end{comment}

\section{Conclusions}

We have corrected and generalized the theory and algorithms for the
generic construction of solutions in valuation based systems. To the
best of our knowledge, these results provide the most general theory
for dynamic programming up-to-date, covering commonly used examples
such as finding the maximum of a combination of sparse functions,
which the current theory did not cover.

\bibliographystyle{plain}
\bibliography{library}
\-
\end{document}